%% file: paper.tex
\documentclass{article}

\input{include}

\begin{document} 

\twocolumn[
\icmltitle{Fast Approximate Spectral Clustering for Dynamic Networks}

\begin{icmlauthorlist}
\icmlauthor{Lionel Martin}{epfl}
\icmlauthor{Andreas Loukas}{epfl}
\icmlauthor{Pierre Vandergheynst}{epfl}
\end{icmlauthorlist}

\icmlaffiliation{epfl}{\'Ecole Polytechnique F\'ed\'erale de Lausanne, Lausanne, Switzerland}

\icmlcorrespondingauthor{Lionel Martin}{lionel.martin@epfl.ch}

\icmlkeywords{clustering, spectral clustering, dynamic clustering, graph signal processing, graph filtering, compressive spectral clustering, eigendecomposition, time series analysis, machine learning, ICML}

\vskip 0.3in
]

\printAffiliationsAndNotice{}

\begin{abstract} 

Spectral clustering is a widely studied problem, yet its complexity is prohibitive for dynamic graphs of even modest size.
%
We claim that it is possible to reuse information of past cluster assignments to expedite computation. Our approach builds on a recent idea of sidestepping the main bottleneck of spectral clustering, i.e., computing the graph eigenvectors, by using fast Chebyshev graph filtering of random signals.
%
We show that the proposed algorithm achieves clustering assignments with quality approximating that of spectral clustering and that it can yield significant complexity benefits when the graph dynamics are appropriately bounded. 

\end{abstract} 

\section{Introduction}

Spectral clustering (SC) is one of the most well-known methods for clustering multivariate data, with numerous applications in  biology (e.g., protein-protein interactions, gene co-expression) and social sciences (e.g., call graphs, political study) among others~\cite{von2007tutorial,fortunato2010community}.
However, because of its inherent dependence on the spectrum of some large graph, SC is also notoriously slow. This has motivated a surge of research focusing in reducing its complexity, for example using matrix sketching methods~\cite{fowlkes2004spectral,li2011time,gittens2013approximate} and more recently compressive sensing techniques~\cite{ramasamy2015compressive,tremblay2016compressive}.

Yet, the clustering complexity is still problematic for dynamic graphs, where the edge set is a function of time. Temporal dynamics constitute an important aspect of many network datasets and should be taken into account in the algorithmic design and analysis. Unfortunately, SC is poorly suited to this setting as eigendecomposition --its main computational bottleneck-- has to be recomputed from scratch whenever the graph is updated, or at least periodically~\cite{ning2007incremental}. This is a missed opportunity since the clustering assignments of many real networks change slowly with time, suggesting that successive algorithmic runs wastefully repeat similar computations.

Motivated by this observation, this paper proposes an algorithm that reuses information of past cluster assignments to expedite computation. 
Different from previous work on dynamic clustering, our objective is \textit{not} to improve the clustering quality, for example by enforcing a temporal-smoothness hypothesis~\cite{chakrabarti2006evolutionary,chi2007evolutionary} or by using tensor decompositions~\cite{gauvin2014detecting,tu2016detecting}. On the contrary, we focus entirely on decreasing the computational overhead and aim to produce assignments that are provably close to those of SC. 

Our work is inspired by the recent idea of sidestepping eigendecomposition by utilizing as features random signals that have been filtered over the graph~\cite{tremblay2016compressive}. Our main argument is that, instead of computing the clustering assignment of a graph $G_1$ using $d$ filtered signals as features, one may utilize a percentage of features of a different graph $G_2$ without significant loss in accuracy, as long as $G_1$ and $G_2$ are appropriately close. This leads to a natural clustering scheme for time-varying topologies: each new instance of the dynamic graph is clustered using $pd$ signals computed previously and only $(1-p)d$ new filtered signals, where $p$ is a percentage.
Moreover, inspired by similar ideas we can also attain further complexity reductions with respect to the graph filter design, i.e., by identifying the $k$-th eigenvalue.        


Concretely, we provide the following contributions:

\textit{1.\,In Section~\ref{sec:static} we refine the analysis of compressive spectral clustering (CSC) presented in~\cite{tremblay2016compressive}}. Our goal is to move from assertions about distance preservation to guarantees about the quality of the solution of CSC itself. We prove that with probability at least $1 - \exp{(-t^2/2)}$, the quality of the clustering assignments of CSC and SC differ by less than $2\sqrt{k/d} \, (\sqrt{k} + t)$. Our analysis suggests that $d = \O(k^2)$ filtered signals are sufficient to guarantee a good approximation, while not making any restricting assumptions about the graph structure, e.g., assuming a stochastic block model~\cite{pydi2017spectral}.

\textit{2.\,In Section~\ref{sec:dynamic}, we focus on dynamic graphs and propose dynamic CSC, an algorithm that reuses information of past cluster assignments to expedite computation.} We discover that the algorithm's ability to reuse features is inherently determined by a metric of spectral similarity $\rho$ between consecutive graphs. Indeed, we prove that, when $pd$ features are reused, the clustering assignment quality of dynamic CSC approximates with high probability that of CSC up to an additive term in the order of $p \rho$.  


\textit{3.\,We complement our analysis with a numerical evaluation in Section~\ref{sec:experiments}.} Our experiments illustrate that dynamic CSC yields in practice computational benefits when the graph dynamics are bounded, while producing assignments with quality closely approximating that of SC.

\section{Background}
\label{sec:background}

We start by  briefly summarizing the standard  method for spectral clustering as well as the idea behind the more recent fast (compressive) methods.



\subsection{Spectral clustering (SC)}

\swapversion{
Spectral clustering is a standard algorithm for graph clustering, i.e., for associating nodes of the graph with clusters so that the number of inter- and intra-cluster connections are minimized and maximized, respectively.}{} To determine the best node-to-cluster assignment, spectral clustering entails solving a $k$-means problem, with the eigenvectors of the graph Laplacian $\L$ as features. 

Let $\G = (\V, \E, \W)$ be a simple symmetric undirected weighted graph with a fixed set of vertices $\V = \{v_1, v_2, \dots, v_n \}$ of cardinality $n$, and a set of $m$ edges $\E \subset \V \times \V$ where the edge between $v_i$ and $v_j$ has weight $\W_{i, j} > 0$ if it exists, and $\W_{i, j} = 0$ otherwise. 
Some versions of spectral clustering make use of the combinatorial Laplacian $\L = \D - \W$ and others of the normalized Laplacian $\L = \I- \D^{-1/2} \W \D^{-1/2} $, see e.g.,~\cite{ng2002spectral, shi2000normalized}. Here, $\D$ is the diagonal matrix whose entries are the degree of the nodes in the graph. We denote the eigendecomposition of the Laplacian of choice by $\L = \U \Lam \U^\top$, with the diagonal entries of $\Lam$ sorted in non-decreasing order, such that $0 = \lambda_1 \leq \lambda_2 \leq \ldots \leq \lambda_n$. 

\swapversion{
It follows directly from the relaxation of the (normalized) cut problem  for $k$ classes that the $k$ eigenvectors associated with the smallest eigenvalues of $\L$ form the optimal solution.} Spectral clustering consists \swapversion{thus} of computing the first $k$ eigenvectors of $\L$ arranged in a matrix called $\Uk$ and subsequently computing a $k$-means assignment of the $n$ vectors of size $k$ found in the rows of $\Uk$. 
Formally, if $\bPhi \in \Rbb^{n\times d}$ is the feature matrix (here $\bPhi= \Uk$ and $d = k$), and $k$ is a positive integer denoting the number of clusters, the $k$-means clustering problem finds the indicator matrix $\Xopt{} \in \Rbb^{n\times k}$ which satisfies
\begin{equation}
    \Xopt{\Phi} = \argmin_{\mathbf{X}\in\mathcal{X}} \|\bPhi - \mathbf{X}\mathbf{X}^\top\bPhi\|_F,
    \label{eq:SC_cost}
\end{equation}
with associated cost $\Copt{\Phi} = \|\bPhi - \Xopt{\Phi}\Xopt{\Phi}^T \bPhi\|_F$.
Symbol $\mathcal{X}$ denotes the set of all $n\times k$ indicator matrices $\mathbf{X}$. These matrices indicates the cluster membership of each data point by setting
\begin{equation}
    \mathbf{X}_{i, j} = \begin{cases}
        \frac1{\sqrt{s_j}} &\mbox{if data point } i \mbox{ belongs to cluster } j\\
        0 & \mbox{otherwise,}
    \end{cases}
\end{equation}
where $s_j$ is the size of cluster $j$, also equals to the number of non-zero elements in column $j$. Note that the cost described in eq.~\eqref{eq:SC_cost} is the square root of the more traditional definition expressed with the distances to the cluster centers \citep[][Sec 2.3]{cohen2015dimensionality}.
We refer the reader to the work by Boutsidis et al.~\citeyearpar{boutsidis2015spectral} and its references for more details.
 


\subsection{Compressive spectral clustering (CSC)}

To reduce the cost of spectral clustering, i.e.~$\O(kn^2)$, Tremblay et al.~\citeyearpar{tremblay2016compressive} and Boutsidis et al.~\citeyearpar{boutsidis2015spectral} independently proposed to fasten spectral clustering using approximated eigenvectors based on random signals. The former also introduced the benefits of compressive sensing techniques reducing the total cost down to $\O(k^2 \log^2(k) + m n(\log(n) + k)$, where $m$ is the order of the polynomial approximation. Their argument consists of two steps: 

\emph{Step 1. Approximate features.} \swapversion{The targeted feature matrix $\bPhi= \Uk$ being costly to compute, one can instead approximate it with the projection of random signals over the graph.}{The costly to compute feature matrix $\bPhi= \Uk$ is approximated by the projection of a random matrix over the same subspace.} In particular, let $\R \in \Rbb^{N\times d}$ 
be a random (gaussian) matrix with centered \iid entries, each having variance $\frac1d$. We can project $\R$ onto $\textit{span}\{\Uk\}$ by filtering each one of its columns by a low-pass graph filter $g(\L) = \Hk$ defined as  
\begin{equation}
    \label{eq:Hk}
    \Hk = \U \left(
        \begin{array}{cc}
            \mathbf{I}_k & 0\\
            0 & 0
        \end{array}
    \right) \U^\top.
\end{equation}
It is then a simple consequence of the Johnshon-Lindenstrauss lemma that the rows $\bpsi_i^\top$ of matrix $ \bPsi = \Hk\R$ can act as a replacement of the features used in spectral clustering, \ie the rows $\bphi_i^\top$ of $\bPhi = \Uk$.

\begin{theorem}[adapted from~\cite{tremblay2016compressive}] For every two nodes $v_i$ and $v_j$ the restricted isometry relation 
\begin{equation}
    (1-\varepsilon) \|\bphi_i - \bphi_j\|_2 \leq \|\bpsi_i - \bpsi_j\|_2 \leq(1+\varepsilon) \|\bphi_i - \bphi_j\|_2
\end{equation}
holds with  probability larger than $1 - n^{-\beta}$, as long as the dimension is $d > \frac{4+2\beta}{\varepsilon^2/2 - \varepsilon^3/3}\log(n)$.
\label{theorem:csc}
\end{theorem}

We note that, even though $\Hk \R$ is also expensive to compute, it can be approximated in $O(|\E|d m)$ number of operations using Chebychev polynomials~\cite{shuman2011chebyshev, hammond2011wavelets}, resulting in a small additive error that decreases with the polynomial order. 

\emph{Step 2. Compressive $k$-means.} The complexity is reduced further by computing the $k$-means step for only a subset of the nodes. The remaining cluster assignments are then inferred by solving a Tikhonov regularization problem involving $k$ additional graph filtering operations, each with a cost linear in $m|\E|$.

To guarantee a good approximation, it is sufficient to select $\O(\nu_k^2\log(k))$ nodes uniformly at random, where $\nu_k = \sqrt{n} \max_i {\|\bphi_i\|_2}$ is the global cumulative coherence. However as shown by Puy et al.~\citeyearpar{puy2016random}, it is always possible to sample $\O(k\log(k))$ nodes using a different distribution (variable density sampling).

In the following, we will present our theoretical results with respect to the non-compressed version of their algorithm.

\section{The approximation quality of static CSC} 
\label{sec:static}

\input{section3_b}
\subsection{Practical aspects}
\label{subsec:static3}

The study presented above assumes the use of an ideal low-pass filter $\Hk$ of cut-off frequency $\lambda_k$. In practice however, we opt to use the computationally inexpensive Chebyshev graph filters~\cite{shuman2011chebyshev}, which approximate low-pass responses using polynomials. In this case, the used filter takes the form $\tH = \U h(\bLambda)\U^\top$, where $h(\cdot)$ is a polynomial function acting on the diagonal entries of $\bLambda$. 
Nevertheless, it is not difficult to see that, when the filter approximation is tight, the clustering quality is little affected. 

In particular, letting $\tbPsi = \tH\R$ the feature error becomes 
\begin{equation}
    \|\tbPsi - \bPhi\X{k}{d}\Q\|_F \leq \|\tbPsi - \bPsi\|_F + \|\bPsi - \bPhi\X{k}{d}\Q\|_F.
\end{equation}
We recognize the second term that is exactly the result of Cor.~\ref{cor:approx_static_feat} and focus thus on the first term.
\begin{align*}
    \|\tbPsi - \bPsi\|_F & \leq \|\U(h(\bLambda)- \X{n}{k}\X{k}{n})\U^\top\R\|_F\\
        & = \|(h(\bLambda)- \X{n}{k}\X{k}{n})\, \R\|_F\\
        & \leq \|h(\bLambda)- \X{n}{k}\X{k}{n}\|_2 \, \|\R\|_F
        \stepcounter{equation}\tag{\theequation}\label{eq:static_filtering}.
\end{align*}
%
An extension of Thm.~\ref{thm:cost_CSC}, taking into account filter approximation, can thus be derived where  eq.~\eqref{eq:static_filtering} would read with probability as least $1-\exp(-dt^2/2)$:
\begin{equation}
    \|\tbPsi - \bPsi\|_F \leq \O\left(m^{-m} (\sqrt{n} + t)\right),
\end{equation}
where $m$ is the order of the polynomial, $\|h(\bLambda)- \X{n}{k}\X{k}{n}\|_2$ reduces to the approximation error of a steep sigmoid that can be bounded using \citep[][Proposition 3]{shuman2011distributed} and $\|R\|_F$ is bounded in \citep[][Lemma 1]{laurent2000adaptive}. The details are left out due to space constraints.

We notice that the cost of the approximation of ideal low-pass filter depends directly on the quality of the filter. Indeed, the overall error rises with the discrepancies with respect to the ideal filter as shown in eq.~\eqref{eq:static_filtering}. Interestingly, the determination of $\lambda_k$ is also very important because a correct approximation will reduce the number of non-zero eigenvalues and thus the effect of the approximated filter in the very last term of the same equation. Towards these goals, we refer the readers to \cite{di2016efficient,paratte2016fast} and their respective eigencount techniques that allow to approximate the filter in $\O(sm|\E|\log(n))$ operations where $s$ is the number of required iterations and $m$ the order of the polynomial.

\section{Compressive clustering of dynamic graphs} \label{sec:dynamic}

In this section, we consider the problem of spectral clustering a sequence of graphs. We focus on graphs $\G_t$ where $t \in \{1, \dots, \tau\}$, composed of a static vertex set $\V$ and evolving edge sets $\E_t$. 

Identifying each assignment from scratch (using SC or CSC) is in this context a computationally demanding task, as the complexity increases linearly with the number of time-steps. In the following, we exploit two alternative metrics of similarity between graphs at consecutive time-steps in order to reduce the computational cost of clustering. 

\begin{definition}[Metrics of graph similarity]
Two graphs $\G_{t-1}$ and $\G_{t}$ are:
\begin{itemize}
    \item \textbf{($\rho,k$)-spectrally similar} if the spaces spanned by their first $k$ eigenvectors are almost aligned 
    \begin{align} 
        \| \Hk_{t} - \Hk_{t-1} \|_F \leq \rho.
    \end{align}
    \item \textbf{$\rho$-edge similar} if the edge-wise difference of their Laplacians is less than $\rho$
    \begin{align} 
        \| \L_{t} - \L_{t-1} \|_F \leq \rho.
    \end{align} 
\end{itemize}
\end{definition}
We argue that both metrics of similarity are relevant in the context of dynamic clustering.
Two spectrally similar graphs might have very different connectivity in terms of their detailed structure, but possess similar clustering assignments. On the other hand, assuming that two graphs are edge similar is a stronger condition that postulates fine-grained similarities between them. It is however more intuitive and computationally inexpensive to ascertain.

\subsection{Algorithm}

We now present an accelerated method for the assignment of the nodes of an evolving graph. Without loss of generality, suppose that we need to compute the assignment for $\G_t$ while knowing already that of $\G_{t-1}$ and possessing the features that served to compute it. Our approach will be to provide an assignment for graph $\G_t$ that reuses (partially) the features $\bPsi_{t-1}$ computed at step $t-1$. Let $p$ be a number between zero and one, and set $q = 1-p$. Instead of recomputing $\bPsi_{t}$ from scratch running a new CSC routine, we propose to construct a feature matrix $\bTheta_{t}$ which consists of $dq$ new features (corresponding to $\G_{t}$) and $dp$ randomly selected features pertaining to graph $\G_{t-1}$:
\begin{align*}
    \label{eq:features_f}
    \bTheta_{t} & = \begin{pmatrix}\Hk_{t-1} \R_{dp} & \Hk_{t} \R_{dq} \end{pmatrix}\\
         & = \bPsi_{t-1} \S{dp}{d} + \bPsi_{t} \Sbar{dp}{d}
        \stepcounter{equation}\tag{\theequation}
\end{align*}
where we used the sub-identity matrix $ \S{dp}{d} = \I_{d\times dp} \I_{dp\times d}$ and its complement $\Sbar{dp}{d} = \I_{d\times d} - \S{dp}{d}$.

We noticed that an important part of the complexity of CSC is intrinsic to the determination of $\lambda_k$ (step 1 of their algorithm). We propose to benefit from the dynamic setting to avoid recomputing it at each step. 
We propose to admit that the previous value for $\lambda_k$ is a good candidate for the filter at the next step, use it to filter the new random signals and validate whether it suits the new graph. Indeed, the eigencount method requires exactly the result of the step 5 of our algorithm to determine if $\lambda_k$ was correctly determined. We thus compute the new filtered signals and proceed if the eigencount using the new signals is close enough to $k$. Otherwise, we suggest to use the knowledge of the previous result and perform a dichotomy with this additional knowledge following~\cite{di2016efficient}. The final set of features generated in the eigencount now serves as $\bPsi_t$.

The method is sketched in Algo.~\ref{algo:dynamic}. For simplicity, in the following we set $p \leq 0.5$ such that the reused features always correspond to $\G_{t-1}$ (and not to some previous time-step).

\begin{algorithm}[t]
    \begin{algorithmic}[1]
        \REQUIRE $(\G_1, \G_2, \dots, \G_\tau), p, d$
        \ENSURE $(\Ass_1, \Ass_2, \dots, \Ass_\tau)$
        \STATE Determine $h_k^1$ the filter approximation for $\G_1$
        \STATE Find an assignment $\Ass_1$ for $\G_1$ using CSC and $h_k^1$
        \FOR{$t$ from 2 to $\tau$}
            \STATE Randomly pick $dp$ filtered signals generated on $\G_{t-1}$
            \STATE Generate $dq$ feature vectors by filtering as many random signals on $\G_t$ with $h_k^{t-1}$
            \STATE Compute the eigencount on the features of step 5
            \STATE Refine $h_k^t$ if the eigencount is wrong, else keep $h_k^{t-1}$
            \STATE Combine these two sets of features to find an assignment $\Ass_t$ using CSC and $h_k^t$
        \ENDFOR
    \end{algorithmic}
    \label{algo:dynamic}
\caption{Dynamic Compressive Spectral Clustering}
\end{algorithm}


\paragraph{Complexity analysis.} We describe now the complexity of our method and compare it to that of Compressive Spectral Clustering. For simplicity, we focus in a first step on the aspects that do not involve compression.
Note that the first graph in the time-series is computed following exactly the procedure of CSC. However, starting from the second graph, there are two steps where the complexity is reduced with respect to CSC.
First, the optimization proposed for the determination of $\lambda_k$ avoids computing $s$ steps of dichotomy for every graph. We claim that spectrally similar graphs must possess close spectrum, thus close values for $\lambda_k$. One could then expect to recompute $\lambda_k$ from time to time only and that when doing so, benefit from a reduced number of iterations due to the proximity. We call $S$ the total number of steps that we gain. Since one step costs $\O(m|\E|\log(n))$ the total gain is $\O(mS|\E|\log(n))$.
Second, since we reuse random filtered signals from one graph to the next, the total number of computed random signals will necessarily be reduced compared to the use of $\tau$ independent CSC calls. The gain here is $\O(m|\E|dp)$ per time-step.
Finally, all reductions applied through compression can also benefit to our dynamic method. Indeed, we theoretically showed that reusing features from the past can replace the creation of new random signals. Thus, sampling  the combination of old and new signals can be applied exactly as defined in CSC. Then, the result of the sub-assignment can be interpolated also as defined in \cite{tremblay2016compressive}.

\subsection{Analysis of dynamic CSC}

\begin{figure*}[ht!]
\centering    
\subfigure[]{\label{fig:spec_edges}\includegraphics[width=0.32\textwidth]{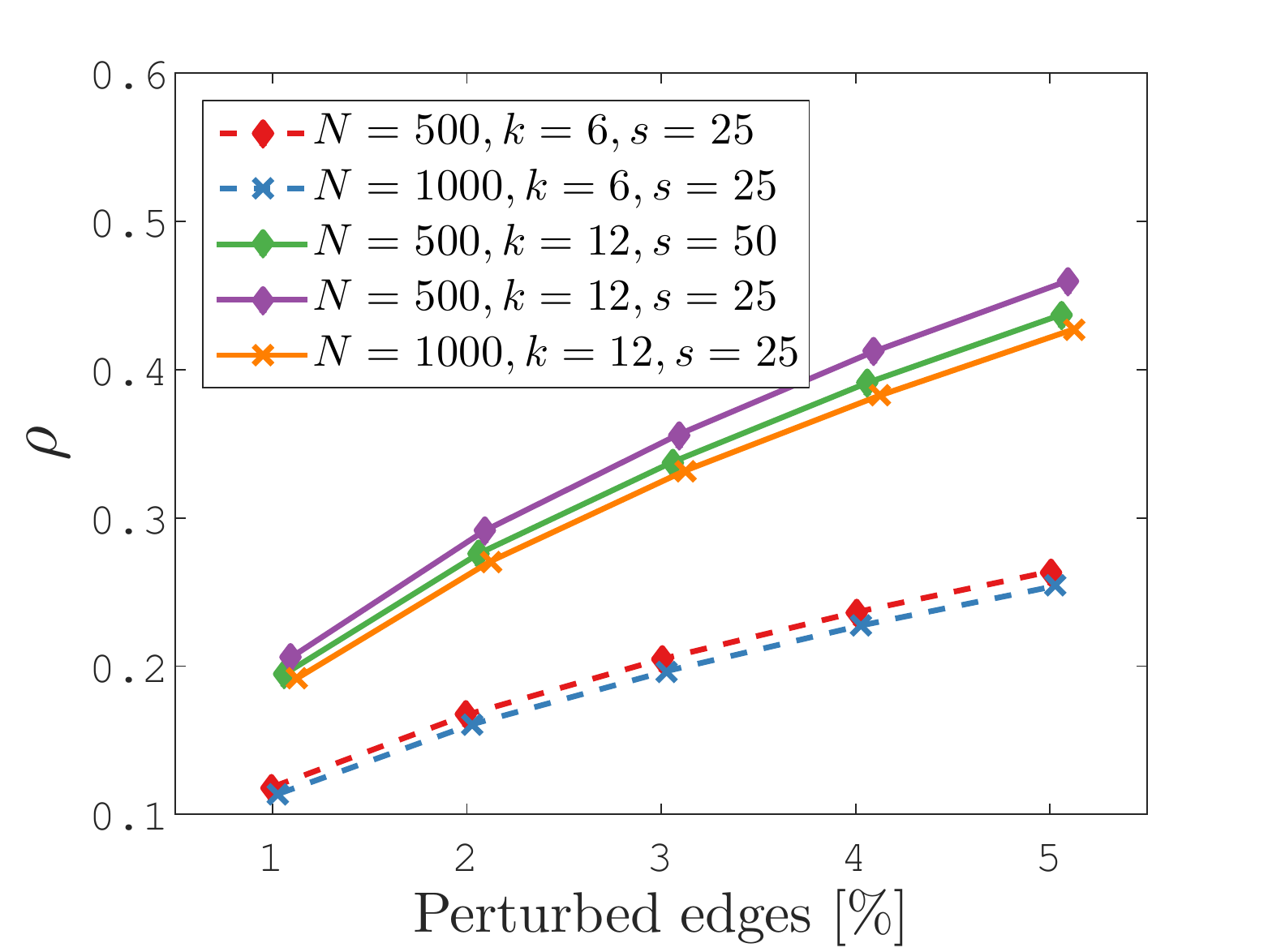}}
\subfigure[]{\label{fig:spec_nodes}\includegraphics[width=0.32\textwidth]{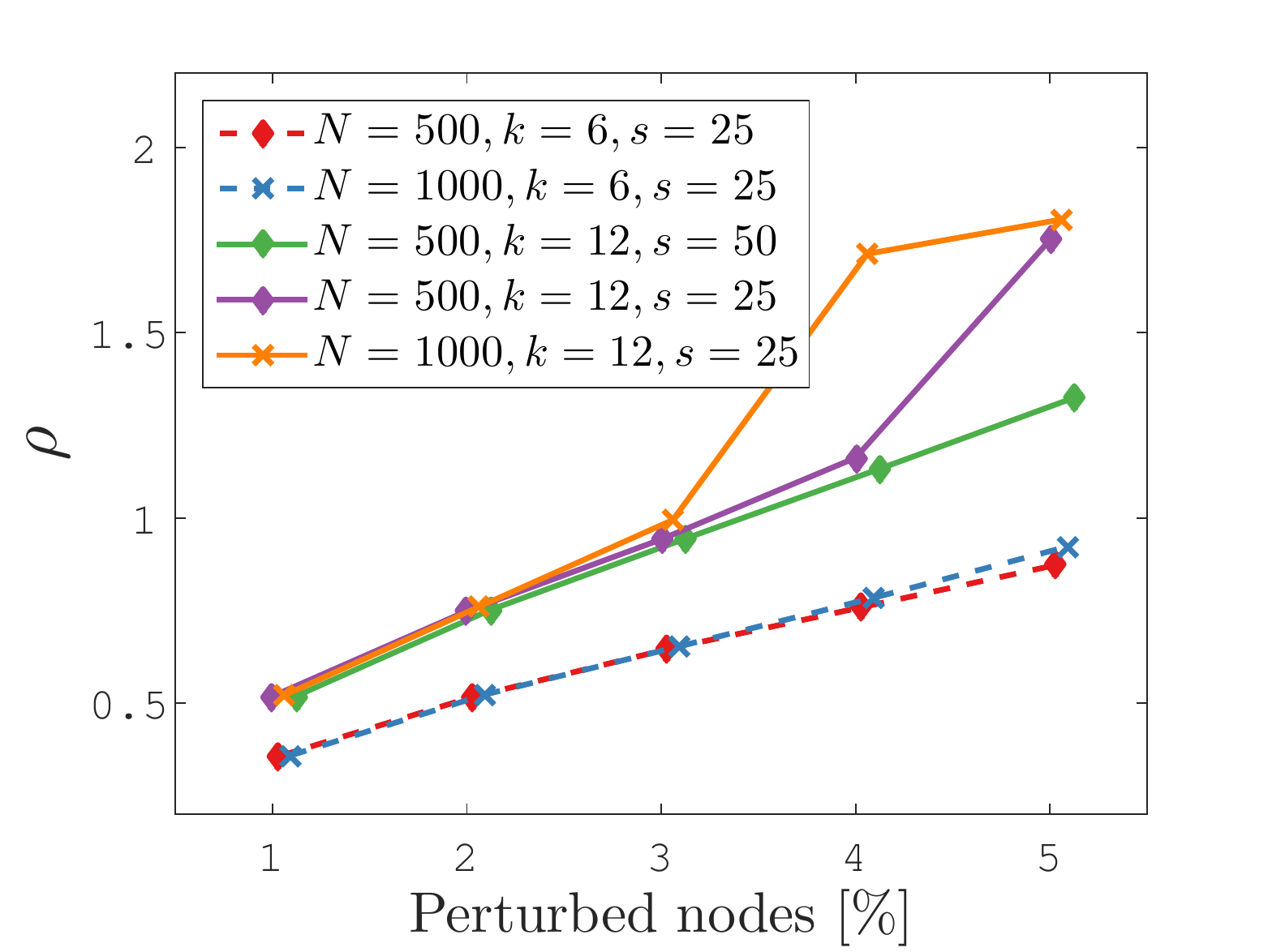}}
\subfigure[]{\label{fig:spec_alpha}\includegraphics[width=0.32\textwidth]{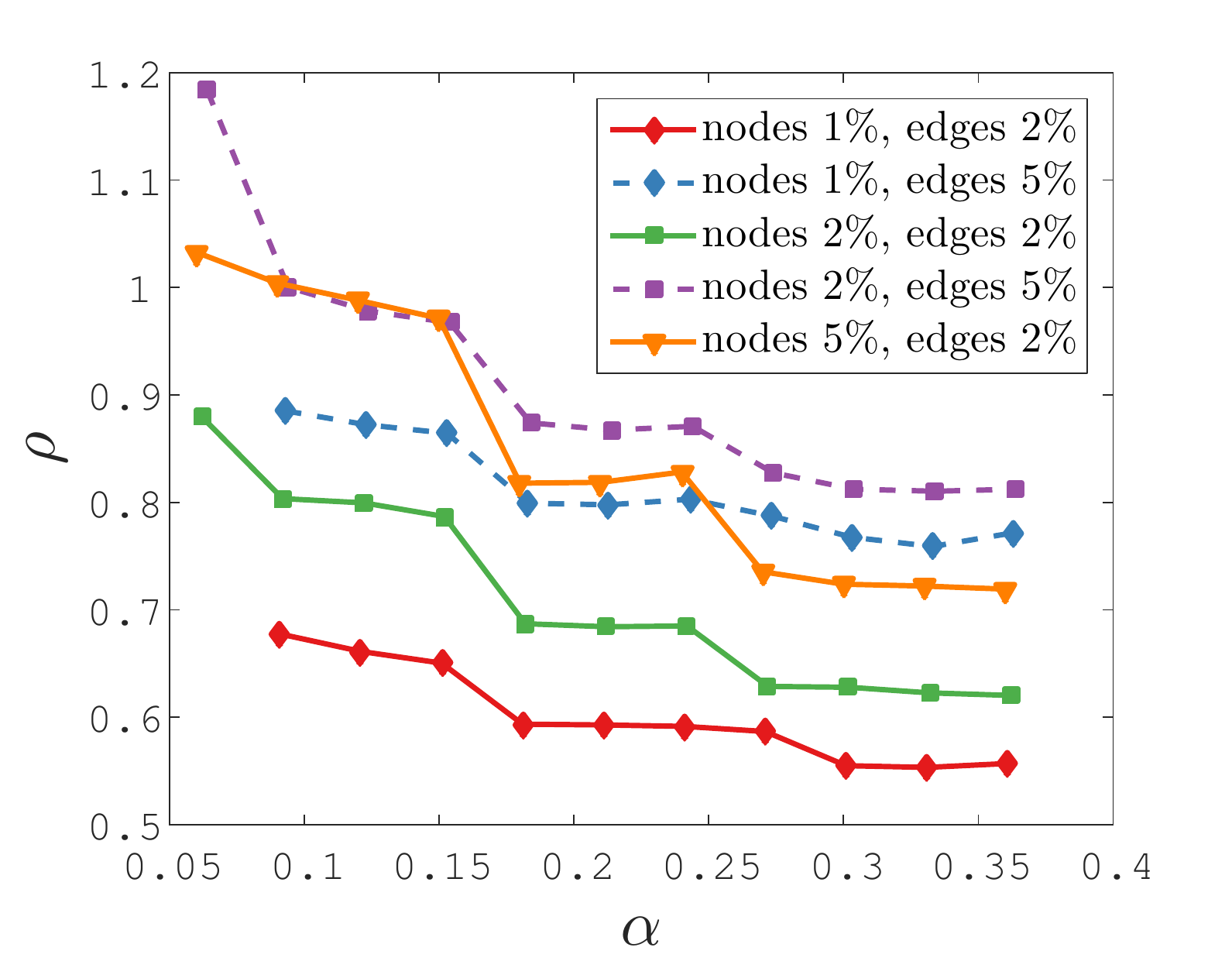}}
\caption{Study of the different perturbation models and their impact on the graph spectral similarity. Graph possessing a large eigengap (highly clusterable) are less subject to perturbations. Proportionally, larger graphs are also less subject to perturbations allowing the number of edge modifications to be larger before a new clustering assignment is required for a given perturbation tolerance. \vspace{0mm}}
\label{fig:spec_sim}
\end{figure*}

Similarly to the static case, our objective is to provide probabilistic guarantees about the approximation quality of the proposed method. Let
\begin{align}
    \Xopt{\Theta_t} = \argmin_{\mathbf{X} \in \mathcal{X}} \|\bTheta_t - \mathbf{X} \mathbf{X}^\top \bTheta_t\|_F.
    \label{eq:DCSC_assignment}
\end{align}
be the clustering assignment obtained from using $k$-means with $\bTheta_t$ as features, and define  the \emph{dynamic CSC cost} $C_{\Theta_t}$ as
\begin{align}
    C_{\Theta_t} = \|\bPhi - \Xopt{\Theta_t} \Xopt{\Theta_t}^\top \bPhi\|_F.
    \label{eq:dyn_CSC_cost}
\end{align}

%
As the following theorem claims, the temporal evolution of the graph introduces \swapversion{now} an additional error term that is a function of the graph similarity (spectral- or edge- wise). 


\begin{theorem}
    \label{thm:cost_dynamic}
    At time $t$, the dynamic CSC cost $C_{\Theta_{t}}$ and the SC cost $C_{\Phi_{t}}$ are related by 
    \begin{align}
       C_{\Phi_t} \leq C_{\Theta_{t}} \leq C_{\Phi_{t}} + 2 \sqrt{\frac{k}{d}} (\sqrt{k} + c) + (1+\delta)p\, \gamma,
    \end{align}
    with probability at least $$1-\exp(-c^2/2) - \exp\left(2\log(n)-dp(\frac{\delta^2}{4}-\frac{\delta^3}{6})\right),$$ where $0 < \delta \leq 1$. Above, $\gamma$ depends only on the similarity of the graphs in question. Moreover, if graphs $\G_{t-1}$ and $\G_{t}$ are
    \begin{itemize}
        \item ($\rho,k$)-spectrally similar, then $ \gamma = \rho $,         
        \item $\rho$-edge similar, then $ \gamma = (\sqrt{2} \, \rho) / \alpha, $ where $\alpha = \min \{\lambda_k^t, \lambda_{k+1}^{(t-1)} - \lambda_k^t\}$ is the Laplacian eigen-gap.
    \end{itemize}
\end{theorem}

\begin{proof}
Let $\Xopt{\Phi_{t}}$ and $\Xopt{\Theta_{t}}$ be respectively the optimal SC and dynamic CSC clustering assignments at time $t$, and denote $\Err = \bTheta_{t} - \bPhi_{t} \X{k}{d} \Q$. 
We have that,
\swapversion{
    \begin{align*}
    C_{\Theta_{t}} & = \|\bPhi_{t} - \Xopt{\Theta_{t}} \Xopt{\Theta_{t}}^\top \bPhi_{t}\|_F\\
        & = \|(\I - \Xopt{\Theta_{t}} \Xopt{\Theta_{t}}^\top) (\bTheta_{t} - \Err)\|_F\\
        & \leq \|(\I - \Xopt{\Theta_{t}} \Xopt{\Theta_{t}}^\top) \bTheta_{t}\|_F + \|\Err\|_F\\
        & \leq \|(\I - \Xopt{\Phi_{t}} \XoptT{\Phi_{t}}) \bTheta_{t}\|_F + \|\Err\|_F\\
        & \leq \|(\I - \Xopt{\Phi_{t}} \XoptT{\Phi_{t}}) (\bPhi_{t} \X{k}{d} \Q + \Err)\|_F + \|\Err\|_F\\
        & \leq \|(\I - \Xopt{\Phi_{t}} \XoptT{\Phi_{t}}) \bPhi_{t} \X{k}{d} \Q\|_F + 2\|\Err\|_F \\
        & = \Copt{\bPhi} + 2\|\bTheta_{t} - \bPhi_{t} \X{k}{d} \Q\|_F.
        \stepcounter{equation}\tag{\theequation}
    \end{align*}
}{
    \begin{equation}
        C_{\Theta_{t}} \leq \Copt{\bPhi} + 2\|\bTheta_{t} - \bPhi_{t} \X{k}{d} \Q\|_F,
    \end{equation}
    following the exact same steps as eq.~\eqref{eq:comparekmeans}.\newline
}
By completing the matrices containing the filtering of both graphs, we can see that the error term can be rewritten as
\begin{align*}
    \| \Err \|_F & = \|\bPsi_{t-1} \S{dp}{d} + \bPsi_{t} \Sbar{dp}{d} - \bPhi_{t} \X{k}{d} \Q\|_F
        \stepcounter{equation}\tag{\theequation}\label{eq:split_pf}\\
        & = \|(\bPsi_{t-1} - \bPsi_{t}) \S{dp}{d} + \bPsi_{t} - \bPhi_{t}\X{k}{d} \Q\|_F\\
        & \leq \|(\bPsi_{t} - \bPsi_{t-1}) \S{dp}{d}\|_F + \|\bPsi_{t} - \bPhi_{t}\X{k}{d}\Q\|_F.
\end{align*}

The rightmost term of eq.~\eqref{eq:split_pf} corresponds to the effects of random filtering and has been studied in depth in Thm.~\ref{thm:error_sing} and Cor.~\ref{cor:approx_static_feat}. The rest of the proof is devoted to studying the leftmost term.

We apply the Johnson-Lindenstrauss lemma~\cite{johnson1984extensions} on the term of interest. Setting $\R'~=~\frac1{\sqrt{p}} \R \X{d}{dp}$, we have that
\begin{align*}
    \|(\bPsi_{t} - \bPsi_{t-1}) \S{dp}{d}\|_F^2 & = \|(\Hk_{t} - \Hk_{t-1}) \R\X{d}{dp}\|_F^2\\
        & = p \, \sum_{i = 1}^n \|\R'^\top\,(\Hk_{t} - \Hk_{t-1})^\top \delta_i\|_2^2.
\end{align*}
Matrix $\R' = {p}^{-1/2} \R \X{d}{dp}$ has $n\times dp$ Gaussian \iid entries with zero-mean and variance ${1/dp}$. It follows from the Johnson-Lindenstrauss lemma that
\begin{align*}
    \|(\bPsi_{t} - \bPsi_{t-1}) \S{dp}{d}\|_F^2 
        & \leq p \, (1+\delta) \sum_{i = 1}^n \|(\Hk_{t} - \Hk_{t-1})^\top \delta_i\|_2^2 \\
        & \leq p \, (1+\delta) \|\Hk_{t} - \Hk_{t-1}\|_F^2,
\end{align*}
with probability at least $1-n^{-\beta}$ and for $dp \geq \frac{4+2 \beta}{\delta^2(\frac12 - \frac{\delta}3)} \log(n)$. Coupling the two together we obtain a probability at least equal to $1-\exp(2\log(n)-\frac{dp\delta^2}{2}(\frac{1}{2}-\frac{\delta}{3}))$, where $\delta$ can be set between 0 and 1. A loose bound gives $2p \|\Hk^{(2)} - \Hk^{(1)}\|_F^2$ with probability $1-\exp(2\log(n)-\frac{dp}{12})$.

This concludes the part of the proof concerning spectrally similar graphs. The result for edge-wise similarity follows from Cor.~\ref{cor:sub_pres}.
\end{proof}

\begin{corollary}[adapted from Cor. 4~\cite{hunter2010performance}]
\label{cor:sub_pres}
Let $\Hk_{t-1}$ and $\Hk_{t}$ be the orthogonal projection on to the span of $[\Uk]_{t-1} (= \bPhi_{t-1})$ and $[\Uk]_{t} (=\bPhi_{t})$. If there exists an $\alpha > 0$ such that $\alpha \leq \lambda_{k+1}^{(t-1)} -\lambda_k^t $ and $\alpha \leq \lambda_k^t$, then,
\begin{equation}
    \|\Hk_{t} - \Hk_{t-1}\|_F \leq \frac{\sqrt{2}}{\alpha} \|\L_t - \L_{t-1}\|_F.
\end{equation}
\end{corollary}
Note that the bounds on $\alpha$ are those described in their Thm.~3.

\begin{figure*}[ht!]
\centering    
\subfigure[]{\label{fig:dyn-a}\includegraphics[width=0.245\textwidth]{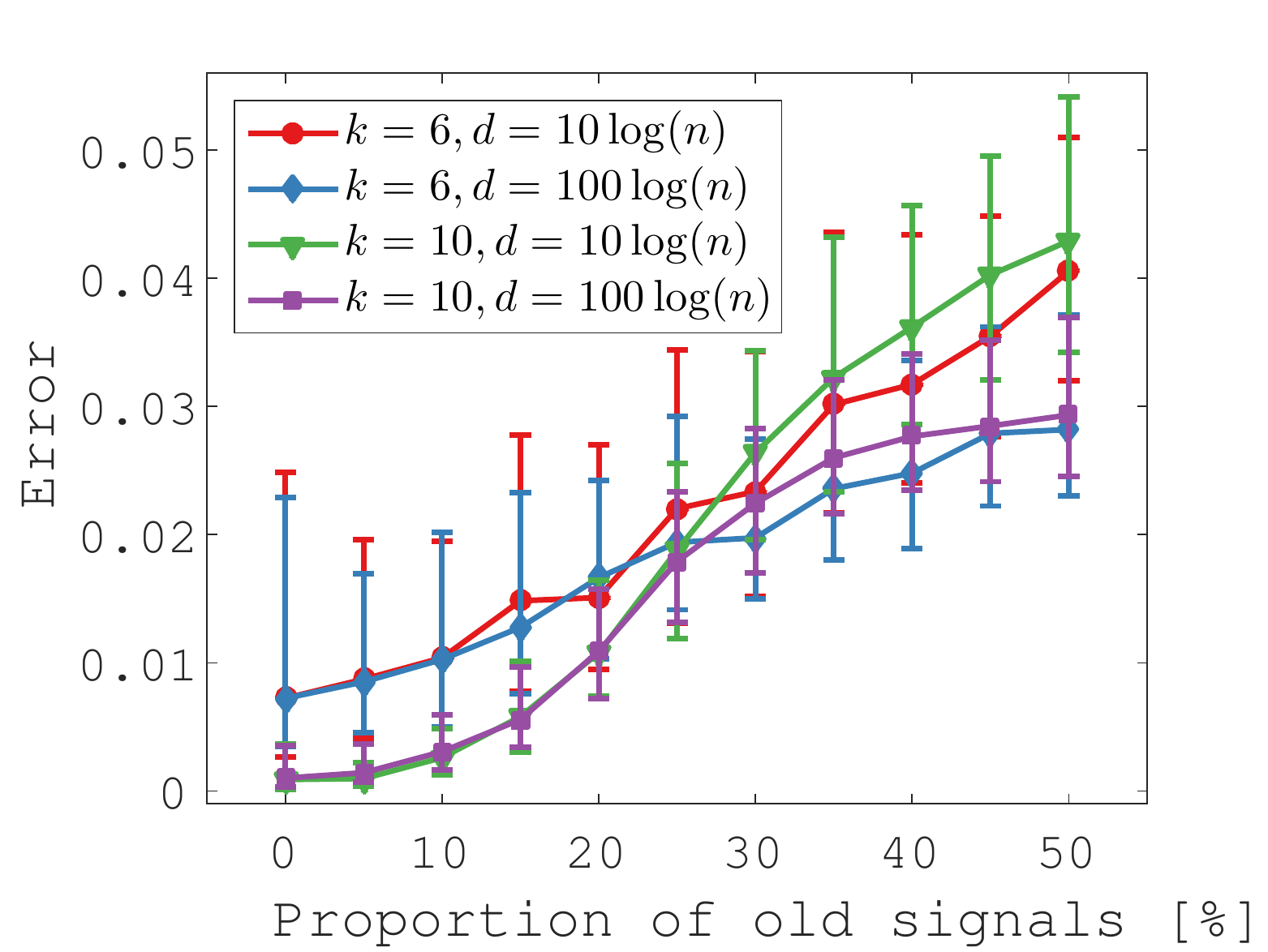}}
\subfigure[]{\label{fig:dyn-b}\includegraphics[width=0.245\textwidth]{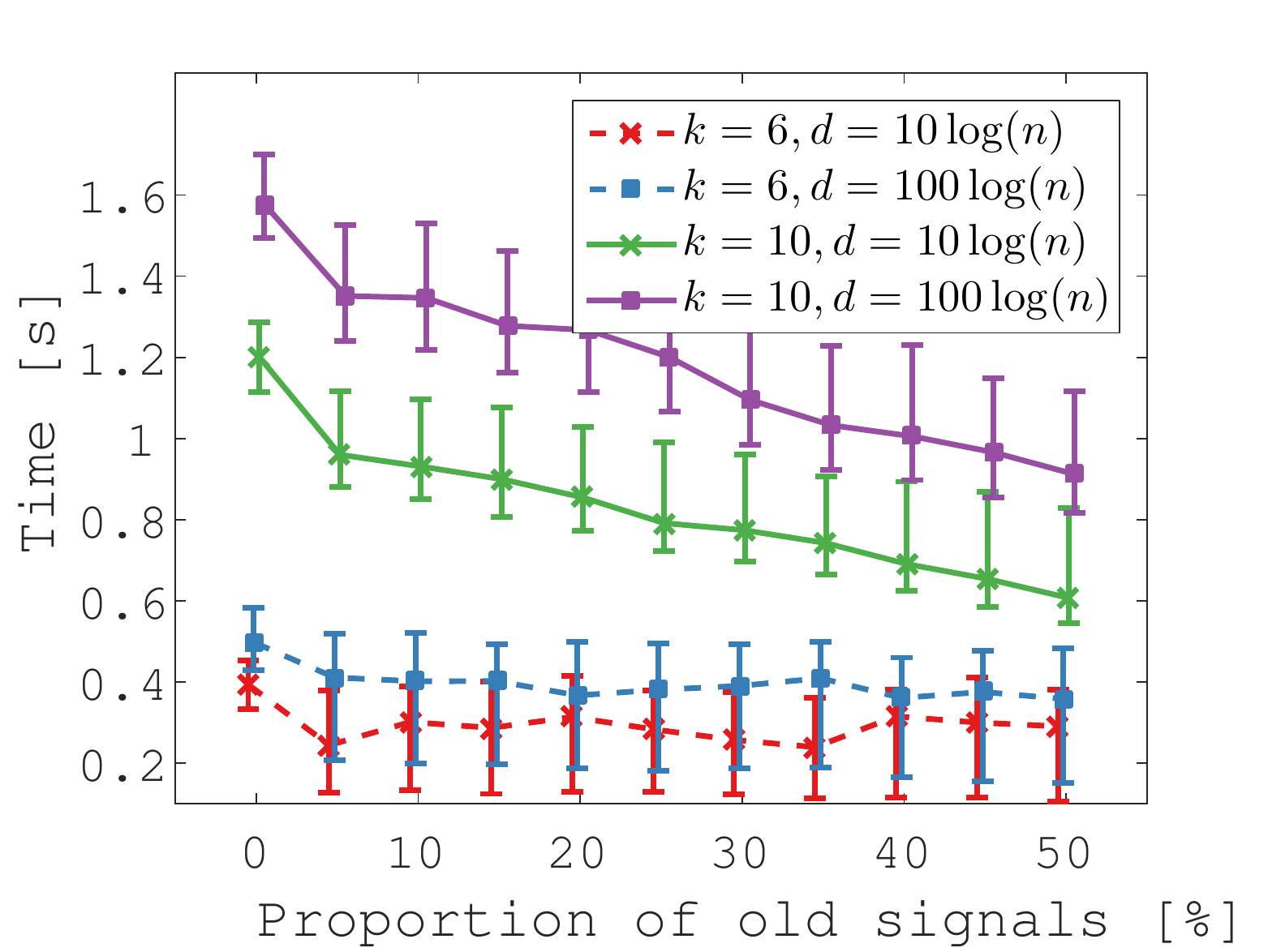}}
\subfigure[]{\label{fig:dyn-c}\includegraphics[width=0.245\textwidth]{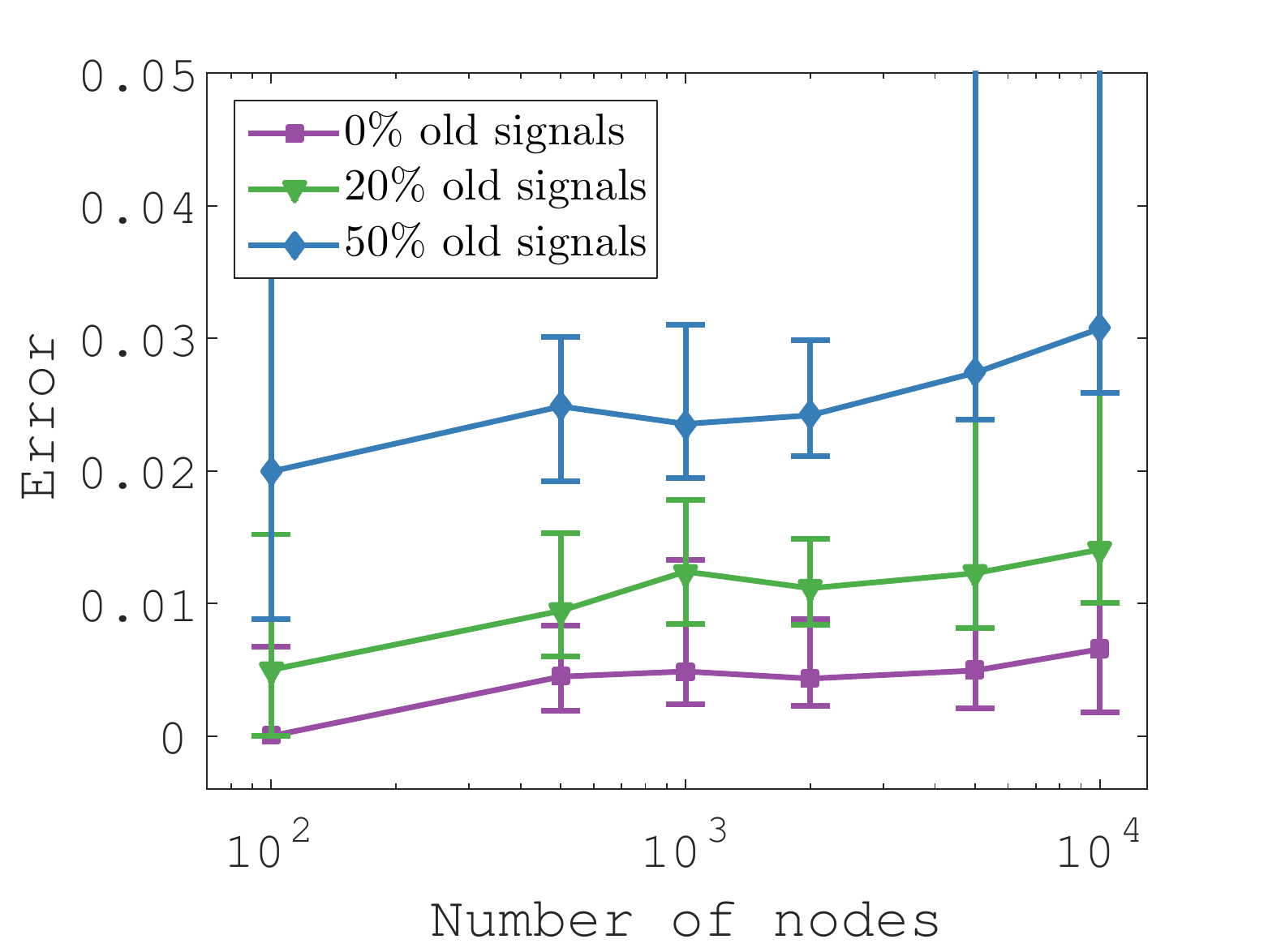}}
\subfigure[]{\label{fig:dyn-d}\includegraphics[width=0.245\textwidth]{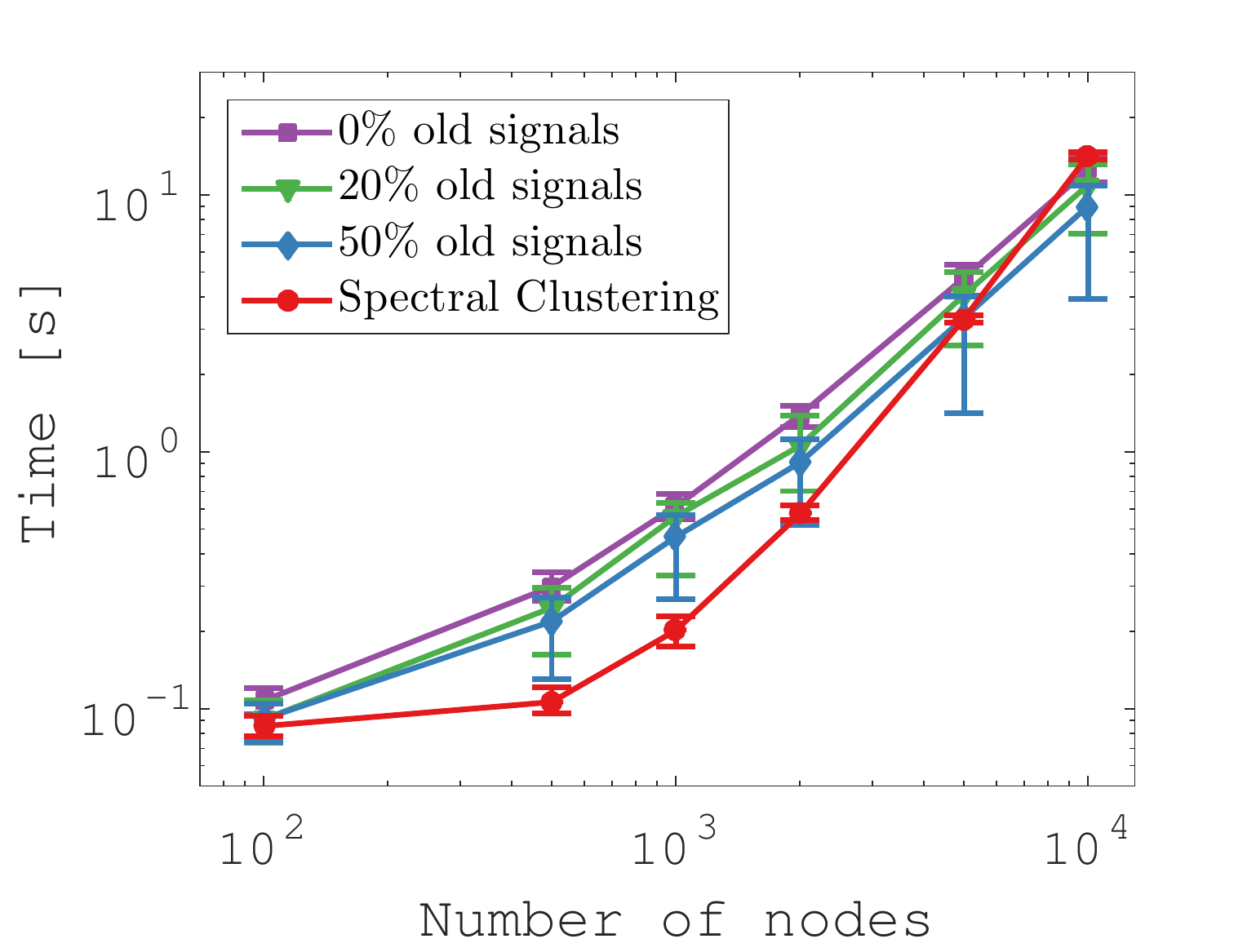}}
\caption{Performances of our algorithm for dynamic graph clustering on synthetic data. Figures (a) and (b) presents the benefits of reusing previous features, while (c) and (d) focus on the scalability of the method with increasing number of nodes. 
\vspace{0mm}}
\label{fig:dyn_syn}
\end{figure*}

\section{Experiments}
\label{sec:experiments}
This section complements the theoretical results described in Section \ref{sec:dynamic}.
First, we study the impact of graph modifications under different perturbation models to the $\rho$-spectral similarity. From there, we present the results of our dynamic clustering algorithm on graphs of different sizes and connectivity.

As is common practice~\citep[e.g.,][]{gorke2013dynamic, tremblay2016compressive} we apply our methods to Stochastic Block Models (SBM). This graph model simulates data clustered into $k$ classes where the $n$ nodes are connected at random with probability for each pair of nodes that depends if the two extremities are belonging to the same cluster ($q_1$) or not ($q_2$), with $q_1 \ll q_2$. In the following, we will qualify the SBM parameters in terms of the node's average degree $s$ and the ratio $e = \frac{q_2}{q_1}$ that represents the graph clusterability.

All our experiments are designed using the GSPBox~\cite{perraudin2014gspbox}.


\subsection{Spectral similarity}
Our theoretical approach highlights the importance of the spectral similarity between two consecutive steps of the graph. We thus start this section by describing how much the graph can change between two assignments. Starting from a SBM, we perform two types of perturbations: edge redrawing and node reassignment. The former simply consists in removing some edges that existed at random and then adding the same number following the probabilities defined by the model. In the latter, one selects nodes instead, removes all edges that share at least one end with the nodes previously picked, assigns those nodes to any other class at random and reconnects these nodes with new edges following the probabilities defined by the graph model.

Figure \ref{fig:spec_sim} shows the similarity of graphs under various perturbation models. Figures \ref{fig:spec_edges} and \ref{fig:spec_nodes} illustrate the impact of the two aforementioned perturbation models separately on SBM of different sizes, whereas in Figure~\ref{fig:spec_alpha} the models are combined. We have three main observations. 

First, the number of clusters $k$ plays a major role in spectral similarilty $\rho$. 
This can be explained by the fact that $\rho$ is bounded by $2\sqrt{k}$. This means that if a similarity threshold is set, one can afford more modifications in the graph when looking for fewer clusters.
Second, we observe that graphs with a larger eigengap remains more similar to the original graph under a given perturbation models, confirming Cor. \ref{cor:sub_pres}.
Finally, it might be also interesting to notice that $\rho$ increases with $n$. This suggests that the algorithm's ability to save computation by reusing information is enhanced for larger graphs.

\subsection{Dynamic clustering of SBM}

We proceed to study the efficiency of dynamic CSC. Based on our previous observations, we set the perturbation model as a combination of the two described in the previous subsection where $1\%$ of the nodes are relabeled and $3\%$ of the edges change. All the results presented here are statistics obtained from simulations replicated 200 times. 

Figure \ref{fig:dyn_syn} displays the results of our clustering for different proportions of previous signals reused in terms of two important metrics: time and accuracy. The error displayed on the figures is the multiplicative error of the $k$-means cost defined in Thm.~\ref{thm:cost_dynamic}, namely $\frac{C_{\Theta_t} - C_{\Phi_t}}{C_{\Phi_t}}$. Since this quantity requires the computation of SC, we are forced to consider problems where $n$ stays in the order of thousands due to its important complexity. While \ref{fig:dyn-a} and \ref{fig:dyn-b} illustrate the benefits of reusing large parts of the previously computed features on graphs with $n=1000, s=25, e=\frac16$, \ref{fig:dyn-c} and \ref{fig:dyn-d} sketch the intuition on large problems with varying values of $n$, setting $k=2\log(n), d=30\log(n)$ while keeping $s=25, e=\frac16$.

First, it is important to notice that as $n$ increases, the time required to perform clustering using CSC methods (including dynamic CSC) outperform that of using SC.
Second, as it could be expected, the error that we observe is slightly increasing as $p$ grows, up to 3\% of the SC cost when reusing 50\% of the previously computed signals. This is very encouraging since in practice, such proportional error is not significant.
Finally, we observe a computational benefit by looking at the time gained by more use of the previous features, as shown in Fig.~\ref{fig:dyn-b}. We emphasize that the improvement in terms of time can attain 25\% of the total time in the most extreme cases depicted in this figure.



\section{Conclusion and Future Work}
The major contribution of this paper is the presentation of a fast clustering algorithm for dynamic graphs that achieves similar quality than Spectral Clustering. We proved theoretically how much the graph can change before losing information for a given computational budget.

We highlighted in this paper several open directions of research for the future. First, it appears clearly in the experiments that the majority of the remaining complexity lies in two steps: the partial $k$-means and the determination of $\lambda_k$. The former is the heart of Spectral Clustering and thus challenging to avoid but seems legitimate to address since there might be various ways to obtain a sub-assignment for some nodes in the graph. The latter, on the opposite, has been already researched in the past although the current methods remain approximated and hamper the results of the filterings.

\bibliographystyle{icml2017}
\bibliography{bibpaper}

\end{document}

%% file: include.tex
\usepackage{times}

\usepackage{graphicx} 
\usepackage{subfigure} 

\usepackage{hyperref}

\usepackage[accepted]{icml2017}

\usepackage{algorithm}
\usepackage{algorithmic}

\usepackage{array}
\usepackage{multirow}
\usepackage{url}
\usepackage{etoolbox}
\usepackage{amsmath,amsthm}
\usepackage{amssymb}
\usepackage{dsfont} 
\usepackage{color}
\usepackage[english]{babel}
\usepackage{comment}
\usepackage{parskip}
\usepackage{bm}
\usepackage{xspace}

\usepackage{chngcntr}
\counterwithin{paragraph}{section}
\counterwithin{paragraph}{subsection}

\renewcommand{\algorithmicrequire}{\textbf{Input:}}
\renewcommand{\algorithmicensure}{\textbf{Output:}}

\newcommand{\Rbb}{\mathbb{R}}

\newcommand{\A}{\mathcal{A}}
\newcommand{\B}{\mathcal{B}}
\newcommand{\E}{\mathcal{E}} 
 
\newcommand{\G}{\mathcal{G}}
\renewcommand{\O}{\mathcal{O}}

\newcommand{\V}{\mathcal{V}}

\newcommand{\D}{\mathbf{D}}
\newcommand{\Err}{\mathbf{E}}

\newcommand{\Hk}{\mathbf{H}}
\newcommand{\I}{\mathbf{I}}
\renewcommand{\L}{\mathbf{L}} 
\newcommand{\Lam}{\mathbf{\Lambda}}
\newcommand{\M}{\mathbf{M}}
\newcommand{\N}{\mathbf{N}}
\newcommand{\Q}{\mathbf{Q}}
\newcommand{\R}{\mathbf{R}}
\newcommand{\Rk}{\mathbf{R}_k}
\newcommand{\U}{\mathbf{U}}
\newcommand{\Uk}{\mathbf{U}_k}
\newcommand{\Ueta}{\mathbf{U}_\eta}
\newcommand{\W}{\mathbf{W}}
\newcommand{\Ass}{\mathbf{X}} 

\newcommand{\ZT}{\mathbf{Z}^\top}
\newcommand{\bLambda}{\mathbf{\Lambda}}
\newcommand{\bPhi}{\mathbf{\Phi}}
\newcommand{\bPsi}{\mathbf{\Psi}}
\newcommand{\bTheta}{\mathbf{\Theta}}
\newcommand{\bSigma}{\mathbf{\Sigma}}
\newcommand{\bphi}{\bm{\phi}}
\newcommand{\bpsi}{\bm{\psi}}
\newcommand{\tbPsi}{\tilde{\bPsi}}
\newcommand{\tH}{\tilde{\mathbf{H}}_k}
\renewcommand{\S}[2]{\mathbf{S}_{#1}^{#2}}
\newcommand{\Sbar}[2]{\overline{\S{#1}{#2}}}

\newcommand{\iid}{i.i.d.~}

\newcommand{\ie}{i.e.,\xspace}

\newcommand{\Copt}[1]{{C}_{#1}}
\newcommand{\X}[2]{\I_{#1\times #2}}
\newcommand{\Xopt}[1]{{\mathbf{X}}_{#1}}
\newcommand{\XoptT}[1]{{\mathbf{X}}^\top_{#1}}

\DeclareMathOperator*{\argmin}{arg\,min}

\newtheoremstyle{style}
{7pt} 
{0pt} 
{\itshape} 
{} 
{\bfseries} 
{.} 
{.5em} 
{} 

\theoremstyle{style}
\newtheorem{theorem}{Theorem}[section]
\newtheorem{definition}{Definition}[section]
\newtheorem{lemma}{Lemma}[section]

\newtheorem{corollary}{Corollary}[section]

\newcommand{\al}[1]{\textcolor{blue}{\textbf{Andreas:} \emph{#1}}}

\newcommand{\swapversion}[2]{#2}

%% file: section3_b.tex
Before delving to the dynamic setting, we refine the analysis of compressive spectral clustering. Our objective is to move from assertions about distance preservation currently known (see Thm.~\ref{theorem:csc}) to guarantees about the quality of the solution of CSC itself. Formally, let 
\begin{align}
    \Xopt{\Psi} = \argmin_{\mathbf{X} \in \mathcal{X}} \|\bPsi - \mathbf{X} \mathbf{X}^\top \bPsi\|_F.
    \label{eq:CSC_assignment}
\end{align}
be the clustering assignment obtained from using $k$-means with $\bPsi$ as features (CSC assignment), and define  the \emph{CSC cost} $C_\Psi$ as
\begin{align}
    C_\Psi = \|\bPhi - \Xopt{\Psi} \Xopt{\Psi}^\top \bPhi\|_F.
    \label{eq:CSC_cost}
\end{align}
The question we ask is: \emph{how close is $C_\Psi$ to the cost $C_{\Phi}$ of the same problem, where the assignment has been computed using $\bPhi$ as features, i.e., the SC cost corresponding to \eqref{eq:SC_cost}?} Note that we choose to express the approximation quality in terms of the difference of clustering assignment costs and not of the distance between the assignments themselves. This has the benefit of not penalizing approximation algorithms that choose alternative assignments of the same quality. 

This section is devoted to the analysis of the quality of the assignments outputted by CSC compared to those of SC for the same graph. 
%
%
Our central theorem, stated below, asserts that with high probability the two costs are close.
\begin{theorem}
    \label{thm:cost_CSC}
    The SC cost $C_\Phi$ and the CSC cost $C_\Psi$ are related by
    %
    \begin{equation}
        C_\bPhi \leq C_\Psi \leq C_\bPhi + 2 \sqrt{\frac{k}{d}} (\sqrt{k} + t),
    \end{equation}
    with probability at least $1-\exp(-t^2/2)$.
\end{theorem}


The result above emphasizes the importance of the number of filtered signals $d$ and directly links it to the distance with the optimal assignment for the spectral features. Indeed, one can see that the difference between the two costs vanishes when $d$ is sufficiently large. Importantly, setting $d = \Omega(k^2)$ guarantees a small error. Keeping in mind that the complexity of CSC is $\O(k^2 \log^2(k) + m n(\log(n) + k)$, we see that our result implies that CSC is particularly suitable when the number of desired cluster is small, e.g.,  $k = O(1)$ or $k = O(\log{n})$.     

\swapversion{
We propose as a side note a multiplicative expression for the error term. Since $C_\Phi^2$ is the cost of the optimal k-means, we can write it as
\begin{equation}
    C_\Phi = \sqrt{\sum_{c:\text{classes}}{\sum_{x\in c}{\|f_x - \mu_c\|_2^2}}} \geq \sqrt{k s},
\end{equation}
where $s$ is the minimal spread of the data points in a class and $k$ the number of classes. Using this fact, we can rewrite the result of Thm.~\ref{thm:cost_CSC} as follows
\begin{equation}
    C_\Psi \leq C_\Phi \left(1 + \frac{2(\sqrt{k}+ t)}{\sqrt{sd}}\right).
\end{equation}

We remark here again that the error can be set arbitrarly small with increasing values of $d$ but also that well clusterable graphs (with small spread $s$) are harder to bound.
}{}

\subsection{The approximation quality of CSC}
\label{subsec:static2}


The first step in proving Thm.~\ref{thm:cost_CSC} is to establish the relation between $C_\Phi$ and $C_\Psi$. The following lemma relates the two costs by an additive error term that depends on the feature's differences $\|\bPsi - \bPhi\X{k}{d}\Q\|_F$\footnote{We assume all along that $d \geq k$ but a similar result holds when $d < k$. In this case, we can consider the term $\|\bPsi\X{d}{k}\Q - \bPhi\|_F$ and derive the optimal unitary $\Q$ in order to obtain the same result as Thm.~\ref{thm:error_sing}. However there is little interest in practice since one cannot expect the recovery of $k$ eigenvectors with less random filtered signals as shown in \cite{paratte2016fast}.}. Since $\bPhi$ and $\bPsi$ have different sizes we introduced the multiplication by a unitary matrix $\Q$. We will first show that any unitary $\Q$ can be picked in Lem.~\ref{lem:csc_error} and then derive the optimal $\Q$, the one minimizing the additive term, in Thm.~\ref{thm:error_sing}.
\begin{lemma}
For any unitary matrix $\Q \in \Rbb^{d\times d}$, the SC cost $C_\Phi$ and the CSC cost $C_\Psi$ are related by
\begin{equation}
    C_\Phi \leq C_\Psi \leq C_\Phi + 2 \|\bPsi - \bPhi \X{k}{d} \Q\|_F,
\end{equation}
where, the matrix $\X{\ell}{m}$ of size $\ell\times m$ above contains only ones on its diagonal and serves to resize matrices.
\label{lem:csc_error}
\end{lemma}
Being able to show that the additive term is small encompasses the result of Thm.~\ref{theorem:csc}, ensuring distance preservation. However, this statement is stronger than the previous one as our lemma is not necessarily true under distance preservation only.



\begin{proof}
    Let $\Xopt{\Phi}$ and $\Xopt{\Psi}$ be respectively the SC and CSC clustering assignments. Moreover, we denote for compactness the additive error term by $\Err = \bPsi - \bPhi \X{k}{d} \Q$.
    We have that
    \begin{align*}
        C_\Psi & = \|\bPhi - \Xopt{\Psi}\XoptT{\Psi} \bPhi\|_F\\
            \swapversion{& = \|(\I - \Xopt{\Psi}\XoptT{\Psi}) \bPhi \X{k}{d} \Q\|_F\\}{}
            & = \|(\I - \Xopt{\Psi}\XoptT{\Psi}) (\bPsi - \Err)\|_F\\
            & \leq \|(\I - \Xopt{\Psi}\XoptT{\Psi}) \bPsi\|_F + \|(\I - \Xopt{\Psi}\XoptT{\Psi}) \Err\|_F\\
            & \leq \|(\I - \Xopt{\Psi}\XoptT{\Psi}) \bPsi\|_F + \|\Err\|_F\\
            & \leq \|(\I - \Xopt{\Phi}\XoptT{\bPhi}) \bPsi\|_F + \|\Err\|_F\\
            & = \|(\I - \Xopt{\Phi}\XoptT{\Phi}) (\bPhi\X{k}{d}\Q + \Err)\|_F + \|\Err\|_F\\
            & \leq \|(\I - \Xopt{\Phi}\XoptT{\Phi}) \bPhi\X{k}{d} \Q\|_F + 2 \, \|\Err\|_F\\
            & = C_\Phi + 2 \, \|\bPsi - \bPhi \X{k}{d} \Q\|_F
            \stepcounter{equation}\tag{\theequation}\label{eq:comparekmeans}
    \end{align*}
     
    The lower bound directly comes from the fact that $\Xopt{\Phi}$ in eq.~\eqref{eq:SC_cost} defines the argmin of our cost functions thus $C_\Phi \leq C_\Psi$.
\end{proof}


The remaining of this section is devoted to bounding the Frobenius error 
$ \|\bPsi - \bPhi \X{k}{d} \Q\|_F$
between the features of SC and CSC. In order to prove this result, we will first express our Frobenius norm exclusively in terms of the singular values of the random matrix $\R$ and then in a second step we will study the distribution of these singular values. 

Our next result, which surprisingly is an equality, reveals that the achieved error is exactly determined by how close a Gaussian matrix is to a unitary matrix.  

\begin{theorem}
\label{thm:error_sing}
There exists a $d\times d$ unitary matrix $\Q$, such that
\begin{equation}
    \label{eq:thm_sing}
    \|\bPsi - \bPhi \X{k}{d} \Q\|_F = \|\bSigma - \X{k}{d}\|_F,
\end{equation}
where $\bSigma$ is the diagonal matrix holding the singular values of  ${\R}' = \X{k}{n}\U^\top\R$.
\end{theorem}

Before presenting the proof, let us observe that ${\R}'$ is an \iid Gaussian random matrix of size $k \times d$ and its entries have zero mean and the same variance as that of $\R$. We use this fact in the following to control the error by appropriately selecting the number of random signals $d$.  


\begin{proof}
    Let us start by noting that, by the unitary invariance of the Frobenius norm, for any $k \times k$ matrix $\M$
    \begin{equation}
        \|\bPhi \M\|_F = \|\U\X{n}{k}\M\|_F = \|\X{n}{k} \M\|_F = \|\M\|_F.
    \end{equation}
    We can thus rewrite the feature error as 
    \begin{align*}
        \|\bPsi - \bPhi \X{k}{d} \Q\|_F & = \|\bPhi\bPhi^\top\R - \bPhi\X{k}{d}\Q\|_F \\
            & = \|\bPhi^\top\R - \X{k}{d}\Q\|_F \\
            & = \|\X{k}{n}\U^\top\R - \X{k}{d}\Q\|_F \\
            & = \|{\R}' - \X{k}{d}\Q\|_F.
            \stepcounter{equation}\tag{\theequation}
    \end{align*}
    We claim that there is a unitary matrix $\Q$ that satisfies eq.~\eqref{eq:thm_sing}. We describe this matrix as follows. Let $ {\R}' ~=~ \Q_L \bSigma \Q_R^\top$ be the singular value decomposition of ${\R}'$ and set
    \begin{equation}
        \Q = \begin{pmatrix} \Q_L & 0 \\ 0 & \I_{d-k}\end{pmatrix}\Q_R^\top.
    \end{equation}
    Substituting this to the feature error, we have that
    \begin{align*}
        \|{\R}' - \X{k}{d}\Q\|_F & = \|\Q_L\bSigma\Q_R^\top - \X{k}{d}\Q\|_F\\
            & = \|\bSigma - \Q_L^\top\X{k}{d}\Q\Q_R\|_F
    \end{align*} \vskip -0.4in
    \begin{align*}
            & = \|\bSigma - \Q_L^\top\X{k}{d}\begin{pmatrix} \Q_L & 0 \\ 0 & \I_{d-k}\end{pmatrix}\Q_R^\top\Q_R\|_F\\
            & = \|\bSigma - \Q_L^\top\begin{pmatrix} \Q_L & 0\end{pmatrix}\|_F\\
            & = \|\bSigma - \X{k}{d}\|_F,
            \stepcounter{equation}\tag{\theequation}
    \end{align*}
    which is the claimed result.
\end{proof}

To bound the feature error further, we will use the following result by Vershynin, whose proof is not reproduced.

\begin{corollary}[adapted from Cor. 5.35~\cite{vershynin2010introduction}]
    Let $\N$ be an $d\times k$ matrix whose entries are independent standard normal random variables. Then for every $t, i \geq 0$, with probability at least $1-\exp(-t^2/2)$ one has
    \begin{equation}
        \sigma_{i}(\N) - \sqrt{d} \leq \sqrt{k} + t, 
    \end{equation}
    where $\sigma_{i}(\N)$ is the $i$th singular value of $\N$.
    \label{cor:Verhynin}
\end{corollary}

Exploiting this result, the following corollary of Thm.~\ref{thm:error_sing} reveals the relation of the feature error and the number of random signals $d$.

\begin{corollary}\label{cor:approx_static_feat}
There exists a $d\times d$ unitary matrix $\Q$, such that, for every $t\geq 0$, one has
    \begin{align}
        \|\bPsi - \bPhi \X{k}{d} \Q\|_F \leq \sqrt{\frac{k}{d}} \, (\sqrt{k}+t),
    \end{align}
    with probability at least $1-\exp(-t^2/2)$. 
\end{corollary}
\begin{proof}
To obtain the following extremal inequality for the singular values of ${\R}'$, we note that ${\R}'$ is composed of \iid Gaussian random variables with zero mean and variance $1/{d}$, and thus use Cor.~\ref{cor:Verhynin} setting ${\R}' = \N/d$ and thus for every $i$,
\begin{align}
    \sigma_{i}({\R}') &= {\sigma_i(\N)}/{\sqrt{d}} \notag \\ 
    &\leq \frac{\sqrt{d} + \sqrt{k} + t}{\sqrt{d}} = 1 + \frac{\sqrt{k} + t}{\sqrt{d}}.
\end{align}
By simple algebraic manipulation, we then find that 
\begin{align*}
    \|\bSigma - \X{k}{d}\|_F^2 & = {\sum_{i=1}^k{ \left(\sigma_i({\R}') - 1 \right)^2}}\\
        &\hspace{-0mm} \leq {k \, \left(\frac{\sqrt{k} + t}{\sqrt{d}}\right)^2} = {\frac{k}{d}} (\sqrt{k}+t)^2,
        \stepcounter{equation}\tag{\theequation}
\end{align*}
which, after taking a square root, matches the claim.
\end{proof}


Finally, Cor.~\ref{cor:approx_static_feat} combined with Lem.~\ref{lem:csc_error} provide the direct proof of Thm.~\ref{thm:cost_CSC} that we introduced earlier.

Before proceeding, we would like to make some remarks about the tightness of the bound. First, guaranteeing that the feature error is small is a stronger condition than distance preservation (though necessary for a complete analysis of CSC). For this reason, the bound derived can be larger than that of Thm.~\ref{theorem:csc}. Nevertheless, we should stress it is tight: the only inequality in our analysis stems from bounding the $k$ largest singular values of the random matrix by Vershynin's tight bound of the maximal singular value.